\def\year{2022}\relax
\documentclass[letterpaper]{article} 

\usepackage{aaai22}  
\usepackage{times}  
\usepackage{helvet}  
\usepackage{courier}  
\usepackage[hyphens]{url}  
\usepackage{graphicx} 
\usepackage{natbib}  
\usepackage{caption} 
\DeclareCaptionStyle{ruled}{labelfont=normalfont,labelsep=colon,strut=off} 
\usepackage{algorithm}
\usepackage{algorithmic}
\usepackage{hyperref}
\usepackage{dblfloatfix}

\usepackage{newfloat}
\usepackage{listings}
\floatstyle{ruled}
\newfloat{listing}{tb}{lst}{}
\floatname{listing}{Listing}

\usepackage{amsmath}
\usepackage{amsfonts}
\usepackage{amsbsy}
\usepackage{amsthm}
\usepackage{xspace}
\usepackage{enumitem}
\usepackage{graphicx}
\usepackage[disable]{todonotes}
\usepackage{color}
\usepackage{multirow}
\usepackage{booktabs}
\usepackage{longtable}

\newtheorem{theorem}{Theorem}

\title{On the Benefits of Biophysical Synapses}
\author{
    Julian Lemmel, 
    Radu Grosu
}
\affiliations{
   Faculty of Informatics of Technische Universit{\"a}t Wien, Austria.\\

   julian.lemmel@tuwien.ac.at, radu.grosu@tuwien.ac.at
}

\begin{document}

\maketitle

\begin{abstract}
The approximation capability of ANNs and their RNN instantiations, is strongly correlated with the number of parameters packed into these networks. However, the complexity barrier for human understanding, is arguably related to the number of neurons and synapses in the networks, and to the associated nonlinear transformations. In this paper we show that the use of biophysical synapses, as found in LTCs, have two main benefits. First, they allow to pack more parameters for a given number of neurons and synapses. Second, they allow to formulate the nonlinear-network transformation, as a linear system with state-dependent coefficients. Both increase interpretability, as for a given task, they allow to learn a system linear in its input features, that is smaller in size compared to the state of the art. We substantiate the above claims on various time-series prediction tasks, but we believe that our results are applicable to any feedforward or recurrent ANN.


\end{abstract}

\section{Introduction}
Inspired by spiking neurons, artificial neurons (ANs) combine in one unit, the additive behavior of biological neurons with the graded nonlinear behavior of their synapses~\cite{bishop1995,gbc2016}. This makes ANs implausible from a biophysical point of view, and precluded their adoption in neural science. 

Artificial neural networks (ANNs) however, correct this biological blunder.
In ANNs it is irrelevant what is the meaning of a neuron, and what is that of a synapse. What matters, is the mathematical expression of the network itself. This was best exemplified by ResNets, which were forced, for technical reasons, to separate the additive transformation from the graded one, and introduce new state variables, which are the outputs of the additive neural, rather than the nonlinear synaptic units ~\cite{heZR016}.

This separation allows us to reconcile ResNets with liquid time-constant neural networks (LTCs), a biophysical model for nonspiking neurons, that shares architectural motifs, such as activation, inhibition, sequentialization, mutual exclusion, and synchronization, with gene regulatory networks~\cite{icra2019,ncp2020,ltc2021,alon2007}. LTCs capture the behavior of neurons in the retina of large species~\cite{Kandel}, and that of the neurons in small species, such as the C.elegans nematode~\cite{wicks1996dynamic}. In LTCs, a neuron is a capacitor, and its rate of change is the sum of a leaking current, and of synaptic currents. The conductance of a synapse varies in a graded nonlinear fashion with the potential of the presynaptic neuron, and this is multiplied with a difference of potential of the postsynaptic neuron, to produce the synaptic current. Hence, the graded nonlinear transformation is the one that synapses perform, which is indeed the case in nature, and not the one performed by neurons.

In contrast to ResNets, NeuralODEs and CT-RNNs~\cite{chen2018neural,funahashi1993approximation}, LTCs multiply (or gate) the conductance with a difference of potential. This is dictated by physics, as one needs to obtain a current. Gating makes each neuron interpretable as a linear system with state-dependent coefficients~\cite{melis18,sdre2008}. Moreover, LTCs associate each activation function to a synapse (like in nature) and not to a neuron (like in ANNs). This allows LTCs to pack considerably more parameters in a network with a given number of neurons and  synapses. As the approximation capability of ANNs and LTCs is strongly correlated with their number of learnable parameters, LTCs are able to approximate the same behavior with a much smaller network, that is explainable in terms of its architectural motifs. We argue that nonlinearity and the size of a neural network are the major complexity barriers for human understanding.

Moving the activation functions to synapses can be accomplished in any ANN, with the same benefits as for LTCs in network-size reduction. The gating of sigmoidal activation functions can be replaced with hyperbolic-tangent activation functions. However, one looses the biophysical interpretation of a neural network, the linear interpretation of its neurons, and the polarity of its synapses.

We compared the expressive power of LTCs with that of CT-RNNs, (Augmented) NeuralODEs, LSTMs, and CT-GRUs, for various recurrent tasks. In this comparison, we considered LTCs and CT-RNNs with both neural and synaptic activation functions. We also investigated the benefits of gating sigmoidal activation with a difference of potential. Our results show that synaptic activation considerably reduces the number of neurons and associated synapses required to solve a task, not only in LTCs but also in CT-RNNs. We  also show that the use of hyperbolic-tangent activation functions in CT-RNNs has similar expressive power as gating sigmoids with a difference of potential, but it looses the linear interpretation.

The rest of the paper is structured as follows. First, we provide a fresh look into ANNs, ResNets, NeuralODEs, CT-RNNs, and LTCs. This paves the way to then show the benefits of biophysical synapses in various recurrent tasks. Finally we discuss our results and touch on future work.

\section{A Fresh Look at Neural Networks} 
\label{sec:dnn}

\subsection{Artificial Neural Networks}
\label{sec:anns}

An AN receives one or more inputs, sums them up in a linear fashion, and passes the result through a nonlinear activation function, whose bias $b$ is the condition for the neuron to fire (spike). However, activation is graded (non-spiking), with smooth (e.g.~sigmoidal) shape. Formally:
\begin{equation}
y_i^{t+1}\,{=}\, \sigma(\sum_{j=1}^n w_{ji}^t\,y_j^{t}\,{+}\,b_i^{t+1})\quad  \sigma(x) \,{=}\,\frac{1}{1+e^{-x}} 
\end{equation}
where as in Figure~\ref{fig:resnet}, $y_i^{t+1}$ is the output of neuron $i$ at layer $t\,{+}\,1$, $y_j^{t}$ is the output of neuron $j$ at layer $t$, $w_{ji}^{t}$ is the weight associated to the synapse between neuron $j$ at layer $t$ and neuron $i$ at layer $t\,{+}\,1$, $b_i^{t+1}$ is the bias (threshold) of neuron $i$ at layer $t\,{+}\,1$, and $\sigma$ is the activation function, e.g.,~the logistic function above. A network with one input layer, one output layer, and $N\,{\geq}\,2$  hidden layers, is called a deep neural network (DNN)~\cite{gbc2016}. ANNs are universal approximators.

Although ANs are biophysically implausible, ANNs are in fact closely related to nonspiking neural networks. To demonstrate this, let us look first at ResNets~\cite{heZR016}.

\subsection{Residual Neural Networks}
\label{ssec:resnets}

DNNs with a large number of hidden layers suffer from the degradation problem, which persists even if the vanishing gradients are curated. Intuitively, DNNs cannot accurately learn identities. Hence, they were simply added to the DNNs in form of skip connections~\cite{heZR016}. 

The resulting architecture, as shown in Figure~\ref{fig:resnet}, was called a residual neural network (ResNet)\footnote{In~\cite{heZR016}, $x_i^t$ skips the first sum and it is added directly to $x_i^{t+2}$. Hence, the architecture shown in Figure~\ref{fig:resnet}, can be regarded as ResNets with finest skip granularity.}. In ResNets, the outputs $x_i^t$ of the sums are distinguished from the outputs $y_i^t$ of the sigmoids. Formally:
\begin{equation}
x_i^{t+1} = x_i^t + \sum_{j=1}^n w_{ji}^t\,y_j^{t} \qquad y_j^{t} = \sigma(x_j^{t}\,{+}\,b_j^t)
\label{eq:resnet1}
\end{equation}

\begin{figure}[t]
    \centering
    \includegraphics[width=0.4\textwidth]{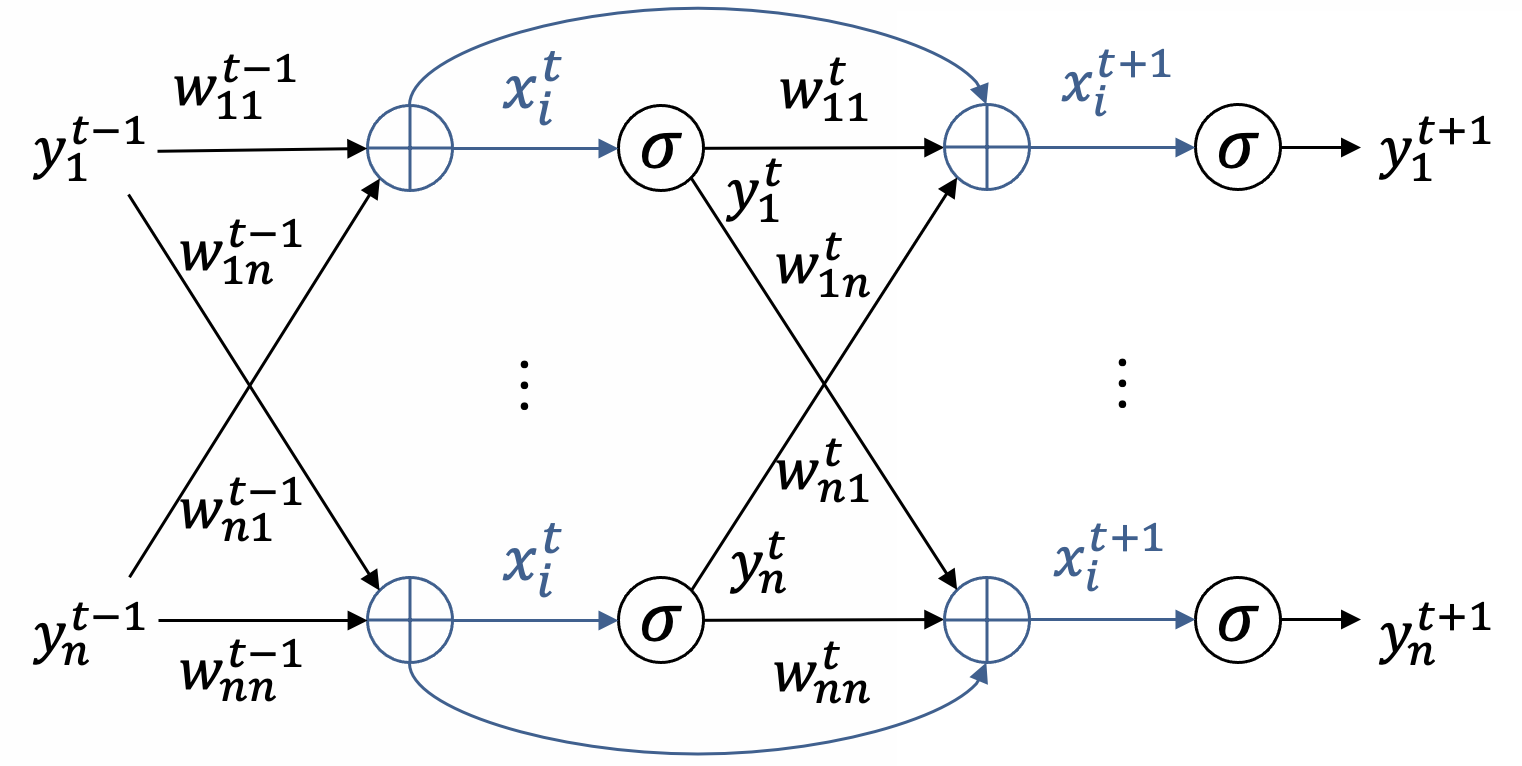}
    \caption{DNN (in black) and ResNet (in black and blue).}
    \label{fig:resnet}
    \vspace*{-3ex}
\end{figure}
\noindent{}This distinction is very important from a biophysical point of view. The main idea is that neurons are just summation units, and the sigmoidal transformation happens in synapses. In fact, one can put the weights in the synaptic transformation, too, which leads to the equivalent equations:
\begin{equation}
x_i^{t+1} = x_i^t + \sum_{j=1}^n y_{ji}^{t} \qquad y_{ji}^{t} = w_{ji}^t\,\sigma(x_j^{t}\,{+}\,b_j^t) 
\label{eq:resnet2}
\end{equation}
Here $w_{ji}^t$ can be thought of as the maximum conductance of the input dependent synaptic transformation $\sigma(x_j^{t}\,{+}\,b_j^t)$. This transformation is indeed graded in nature, that is nonspiking. Since ResNets are particular DNNs, with identity as a linear activation, they are also universal approximators.

\subsection{Neural Ordinary Differential Equations}
\label{ssec:neuralODEs}

Equations~(\ref{eq:resnet2}) is the Euler discretization of a set of differential equations, where the time step is simply taken to be one~\cite{e17,chen2018neural}. Mathematically:
\begin{equation}
\dot{x}_i(t) = \sum_{j=1}^n y_{ji}(t) \quad y_{ji}(t) = w_{ji}(t)\,\sigma(x_j(t)\,{+}\,b_j(t)) 
\label{eq:contRS}
\end{equation}
In these equations, $x$, $y$, and the parameters $w$ and $b$ change continuously in time. Now suppose we make the parameters constant. Are we still going to have a universal ODE approximator?  The answer is yes, as we will show in next section. The differential equations are as follows:
\begin{equation}
\dot{x}_i(t) = \sum_{j=1}^n\,y_{ji}(t) \quad y_{ji}(t) = w_{ji}\,\sigma(x_j(t)\,{+}\,b_j) 
\label{eq:neuralODEs}
\end{equation}
This is the form of Neural Ordinary Differential Equations (NeuralODEs)~\cite{e17,chen2018neural}.\footnote{Strictly speaking, NeuralODEs $\dot{x}(t)\,{=}\,f(x)$ may have an arbitrary number of neural layers for the function $f$.} Taking the state of the network as the sigmoid $y$ of a sum is equivalent to taking the state as the sum $x$ of sigmoids. 

\begin{theorem}[NeuralODEs]
Let $x$ and $y$ be state vectors. Then $\dot{y}\,{=}\,\sigma(Wy+b)$ is equivalent to $\dot{x}\,{=}\,W\sigma(x+b)$.
\end{theorem}
\begin{proof}
Take $x\,{=}\,Wy$. Then the following holds:
\[
\dot{x} = W\dot{y} = W\sigma(Wy+b) = W\sigma(x+b)
\vspace*{-4.4ex}
\]

\end{proof}

A slight extension called ANODEs is given 
in~\cite{dupont2019augmented}, which embeds the input in the internal state, and projects the state to the outputs $S$, as follows:
\begin{equation}
x(t_0) = [x,0]^T \quad y = \pi_S(x(t_N))
\label{eq:aNode}
\end{equation}
 NeuralODEs are harder to learn than ResNets. For the training purpose, one can use the adjoint equation, and employ efficient numerical solvers~~\cite{e17,chen2018neural}.

\subsection{Continuous-Time Recurrent Neural Networks}
\label{ssec:ctrnns}

\subsubsection{Autonomous case.}
In this form of CT-RNNs, the input is the initial state. Let us call them ACT-RNNs, They extend NeuralODEs with a leading term $-w_{i}x_i(t)$~\cite{funahashi1993approximation}. Their mathematical form is as follows:
\begin{equation}
\begin{array}{c}
\dot{x}_i(t) = -w_{i}x_i(t) + \sum_{j=1}^n\,y_{ji}(t)\\[1mm] 
y_{ji}(t) = w_{ji}\,\sigma(x_j(t)\,{+}\,b_j)
\end{array}
\label{eq:na-act-rnn}
\end{equation}
The leading term brings the system back to the equilibrium state, when no input is available. Hence, a small perturbation is forgotten, that is, the system is stable. Like in NeuralODEs, one can interchange sumation and activation.

\begin{theorem}[ACT-RNNs]
Let $x$ and $y$ be state vectors. Then $\dot{y}\,{=}\,{-}w\,{*}\,y\,{+}\,\sigma(Wy+b)$ and $\dot{x}\,{=}\,{-}w\,{*}\,x\,{+}\,W\sigma(x+b)$ are equivalent ODEs where $*$ is the pointwise product of vectors.
\end{theorem}
\begin{proof}
Let $x\,{=}\,Wy$. Then~\cite{funahashi1993approximation}:
\[
\begin{array}{l@{\,{=}\,}l}
\dot{x} & W\dot{y}\\ & W({-}w\,{*}\,y\,{+}\,\sigma(Wy{+}b))\,{=}\,{-}w\,{*}\,x\,{+}\,W\sigma(x\,{+}\,b)
\end{array}
\vspace*{-4ex}
\]
\end{proof}

\noindent{}ACT-RNNs are universal approximators, and stabilization is not relevant in this respect~\cite{funahashi1993approximation}. Hence, NeuralODEs are universal approximators, too.

\subsubsection{\em Synaptic activation.}
Like in ANNs, ACT-RNNs associate each activation function to a neuron. We therefore call them NA-ACT-RNNs, where NA stands for neural activation.

\begin{figure}[t]
    \centering
    \includegraphics[width=0.4\textwidth]{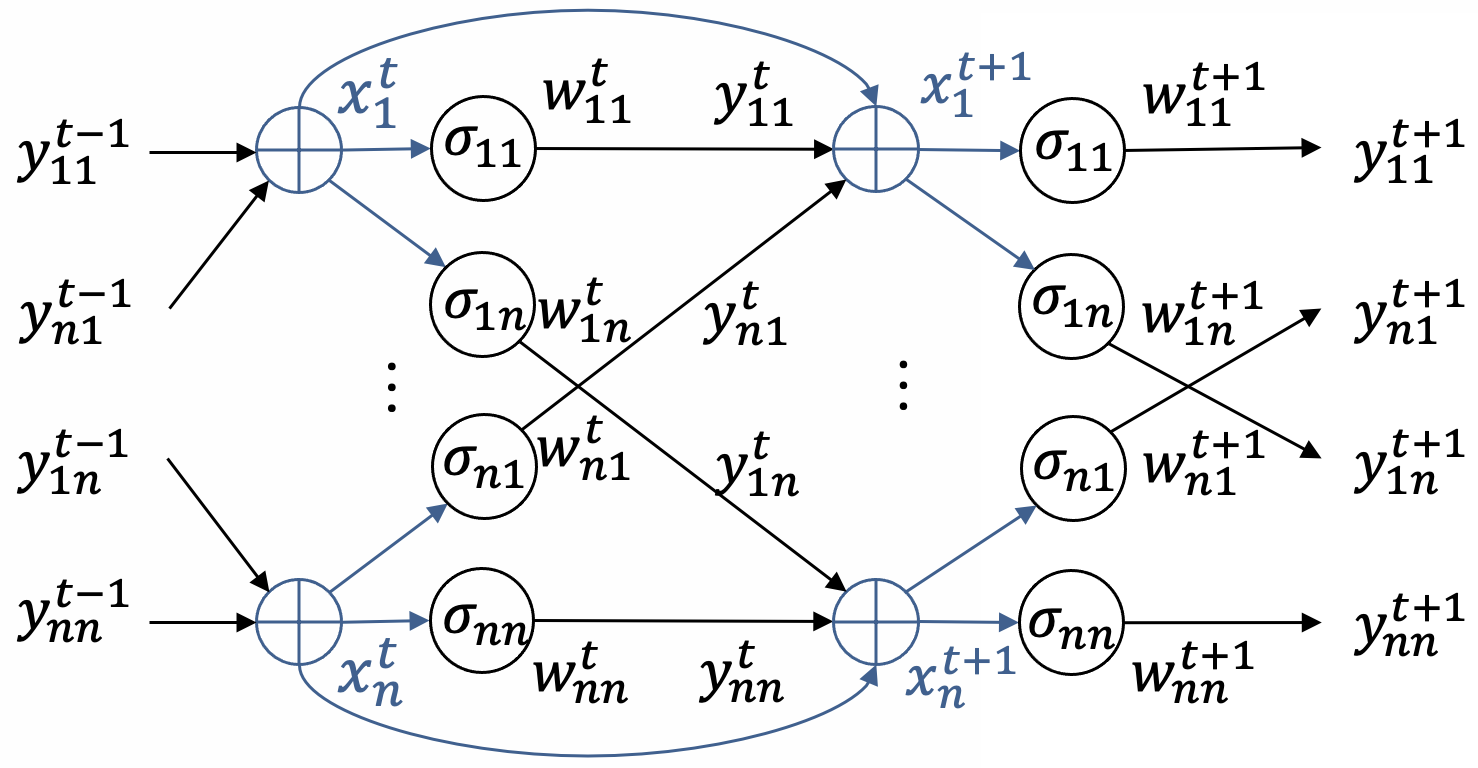}
    \caption{Synaptic-activation DNN and ResNet.}
    \label{fig:synapticResNet}
    \vspace*{-3ex}
\end{figure}

However, as shown in Figure~\ref{fig:synapticResNet}, any ANN can be rewritten, such that activation functions are associated to synapses. We call this form of ACT-RNNs, SA-ACT-RNN, where SA stands for synaptic activation. Adding to each activation a variance $a$, too, one has the following explicit form:
\begin{equation}
\begin{array}{c}
\dot{x}_i(t) = -w_{i}x_i(t) + \sum_{j=1}^n\,y_{ji}(t)\\[1mm] 
y_{ji}(t) = w_{ji}\,\sigma(a_{ji}x_j(t)\,{+}\,b_{ji})
\end{array}
\label{eq:sa-act-rnn}
\end{equation}
The advantage of SA-ACT-RNNs is that they pack many more parameters in a network, for the same number of neurons and synapses.  For example, an SA-ACT-RNN with 32 neurons, connected in an all to all fashion, is able to pack 3104 parameters. This roughly corresponds to an NA-ACT-RNN with 54 neurons which packs 3132 parameters.

While the approximation capability of a neural network is strongly correlated with its number of parameters, we strongly believe that the complexity barrier for human understanding, is in the number of neurons and synapses.

\subsubsection{General case.}
CT-RNNs have in general associated a time varying input signal $u$, too, that is, they are RNNs. The way the input is considered, plays a very important role.

A popular way of adding the input signal $u$, is to extend the sum within a sigmoid with a sum corresponding to the input. In vectorial form this looks as follows:
\begin{equation}
\label{eq:na-ctrnn1}
\dot{y}\,{=}\,{-}w\,{*}\,y\,{+}\,\sigma(Wy+Vu+b)
\end{equation}
This form has excellent convergence properties, but it cannot be extended to synaptic activations. We therefore prefer the following form, which has the same convergence properties:
\begin{equation}
\label{eq:na-ctrnn2}
\dot{x}\,{=}\,{-}w\,{*}\,x\,{+}\,W\sigma(a^x\,{*}\,x+b^x){+}\,V\sigma(a^u\,{*}\,u+b^u)
\end{equation}
where $a^x,b^x$ and $a^u,b^u$ represent the variance and the bias vectors for the state and the input vectors, respectively.  Finally, another popular way of adding the input to CT-RNNs is in a linear fashion, as below. 
%
\begin{equation}
\label{eq:na-ctrnn3}
\dot{x}\,{=}\,{-}w\,{*}\,x\,{+}\,W\sigma(a\,{*}\,x+b){+}\,Vu
\end{equation}
\begin{figure}[t]
    \centering
   \includegraphics[width=0.4\textwidth]{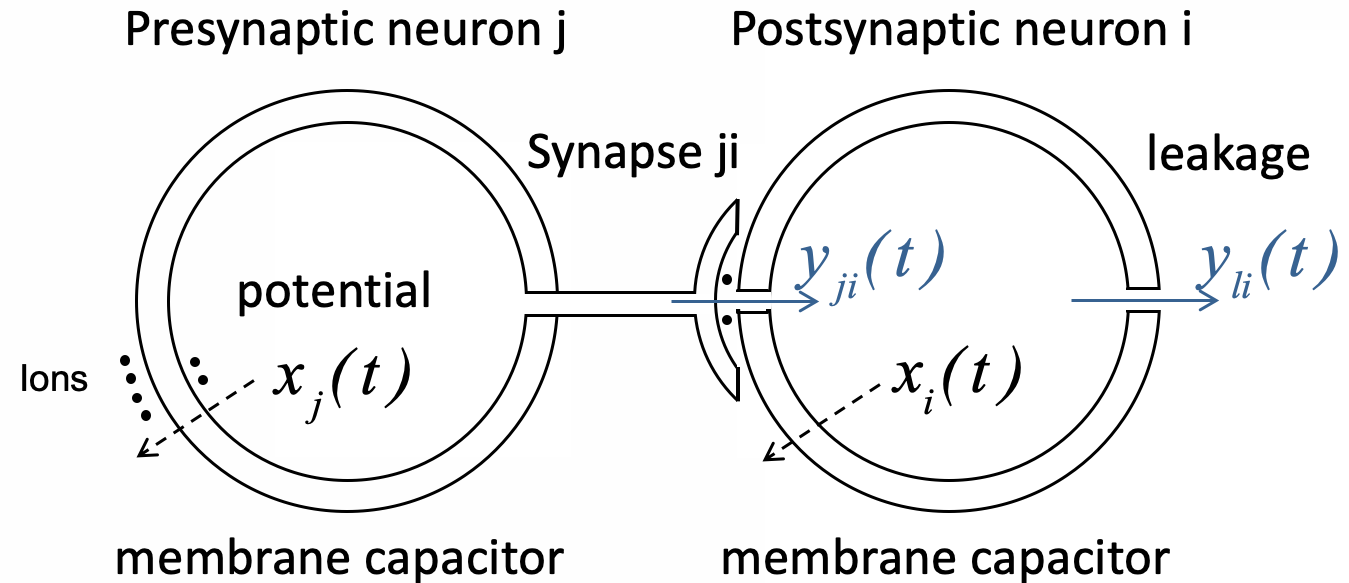}
   \caption{The electric representation of a nonspiking neuron.}
    \label{fig:NRN}
    \vspace*{-3ex}
\end{figure}

\begin{figure*}[t]
    \centering
     \includegraphics[width=.6\textwidth]{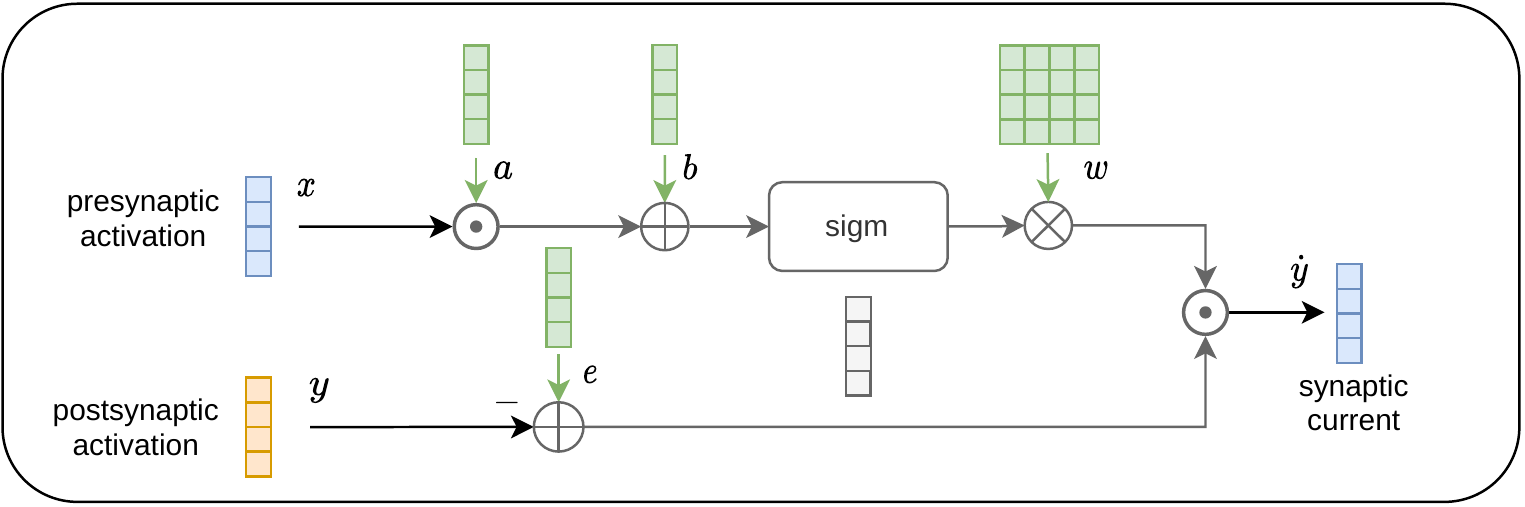}
  \includegraphics[width=.6\textwidth]{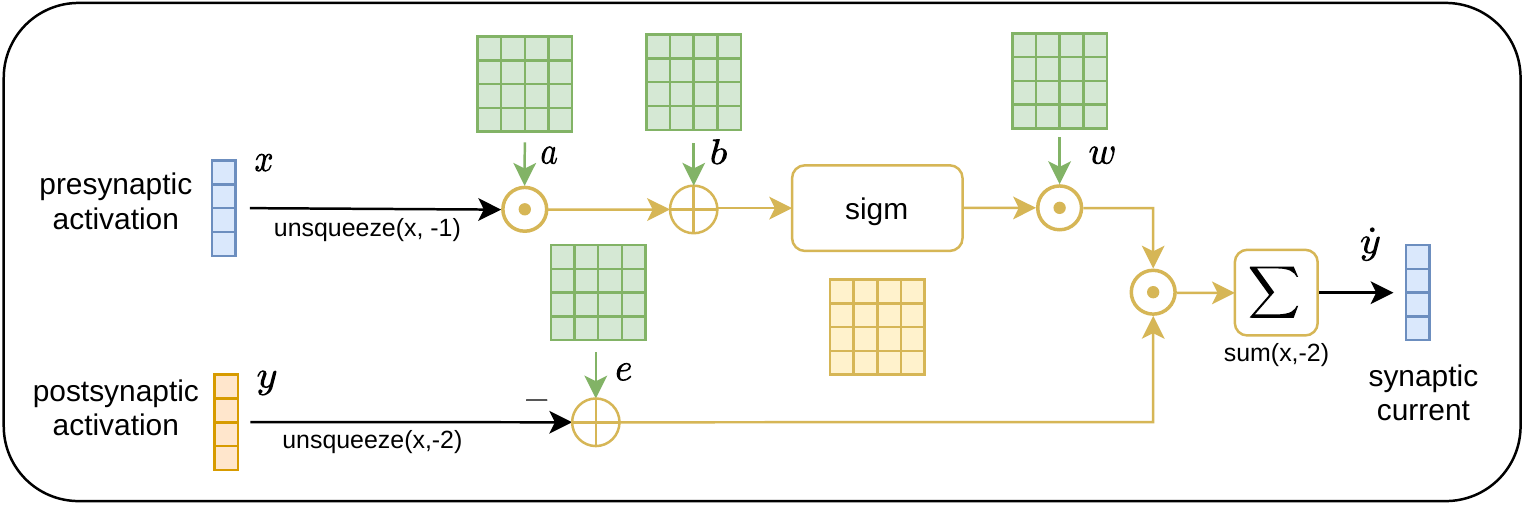}
   
  \caption{Synaptic Layer in LTCs with synaptic (top) and neural (bottom) activation. }
    \label{fig:synaptic}
    \vspace*{-4ex}
\end{figure*}


\subsubsection{\em Synaptic activation.}
Like in SA-ACT-RNNs, the last two CT-RNNs can be rewritten, by associating each activation to a synapse. To distinguish the two variants, we call them NA-CT-RNNs and SA-CT-RNNs, respectively. In scalar form, the sigmoidal-input version can be written as below. The linear-input version is very similar:
\begin{equation}
\begin{array}{c}
\dot{x}_i(t) = -w_{i}x_i(t) + \sum_{j=1}^n\,y_{ji}(t) + \sum_{j=1}^m\,z_{ji}(t)\\[1mm] 
y_{ji}(t) = w_{ji}\,\sigma(a^x_{ji}x_j(t)+b^x_{ji})\\[1mm]
z_{ji}(t) = v_{ji}\,\sigma(a^u_{ji}u_j(t)+b^u_{ji})
\end{array}
\label{eq:ctrnn}
\end{equation}
Now consider an SA-CT-RNN with 32 neurons, connected all to all, and with 32 inputs. It packs a total of 6176 parameters. This roughly corresponds to an NA-CT-RNN with 54 neurons, which packs a total of 6102 parameters.

\subsection{Liquid Time Constant Networks}
\label{sec:NRNs}
\subsubsection{Autonomous case.}
LTCs are a biophysical model for the neural system of small species~\cite{wicks1996dynamic,icra2019,ncp2020,ltc2021}, and the retina of large species~\cite{Kandel}. Due to the small dimension of these neural systems (${\leq}\,1\,$mm), neural transmission happens passively, in the analog domain, without considerable attenuation. Hence, the neurons do not need to spike for an accurate signal transmission. 

As shown in Figure~\ref{fig:NRN}, the neuron's membrane is an insulator, with ions both on its inside and outside. Electrically, it is a capacitor. The difference between the inside-outside ionic concentrations defines the membrane potential (MP) $x$. The rate of change of $x$ depends on the currents $y$ passing through the membrane. These are either external currents (ignored for ALTCs), a leakage current, and synaptic currents. For simplicity, we consider only chemical synapses. The capacitor equation is then as follows:
\begin{equation}
\begin{array}{c}
C\,\dot{x}_i(t) = w_{li}\,(e_{li}-x_i(t)) + \sum_{j=1}^n y_{ji}(t) \\[1mm] 
y_{ji}(t) = w_{ji}\,\sigma(a_{ji}x_j(t)\,{+}\,b_{ji})\,(e_{ji} - x_i(t)) 
\end{array}
\label{eq:sa-altc}
\end{equation}
where $C$ is the membrane capacitance,  $e_{li}$ the resting potential, $w_{li}$ the leaking conductance, and $e_{ji}$ the synaptic potentials. These are either 0\,mV for excitatory synapses (potential $x_i$ is negative so the current is positive), or -90\,mV for inhibitory synapses (the current is in this case negative).  

%
Equations~(\ref{eq:sa-altc}) are very similar to Equations~(\ref{eq:na-act-rnn}) of an SA-ACT-RNN. They have a leaking current, which ensures the stability of the ALTC, and a presynaptic-neuron controlled conductance $\sigma$ for the synapses, with maximum conductance $w_{ji}$. This conductance is multiplied with a difference of potential $e_{ji}\,{-}\,x_i(t)$, to get a current. This biophysical constraint, makes them different from SA-ACT-RNNs. So what is the significance of this gating term from the point of view of machine learning? As we prove below, it has important consequences for the interpretability of ALTCs. 
\begin{theorem}[Interpretability]
Each ALTC neuron is interpretable as linear regression of its inputs.
\end{theorem}
\begin{proof}
Let $x(0)$ be the input. This is propagated in time as $x(t)$. Let $w_{ji}\,\sigma(a_{ji}x_j\,{+}\,b_{ji})$ be the the state-dependent weight from neuron $j$ to neuron $i$. Then according to Equations~(\ref{eq:sa-altc}), $\dot{x}_i(t)$ is a linear regression in $x$, for each $i$. Moreover, small perturbations of $x$ lead to small changes in $\dot{x}$. 
\end{proof}
ALTCs are able to pack even more parameters than SA-ACT-RNNs. For example, an ALTC with 32 neurons, connected in an all to all fashion, is able to pack 
4192 parameters, whereas an SA-ACT-RNN packs only
3104 parameters. This roughly corresponds to an NA-ACT-RNN with 64 neurons which packs 4224 parameters

\subsubsection{\em Neural activation.}
While in nature each synapse has distinct dynamics, one may want to consider that in particular cases, all outgoing synapses of a neuron have the same activation parameters. Let us call this version of ALTCs as NA-ALTCs, where NA stands for neural activation. We also interchangeably refer to ALTCs as SA-ALTCs. Formally:
\begin{equation}
\begin{array}{c}
C\,\dot{x}_i(t) = w_{li}\,(e_{li}-x_i(t)) + \sum_{j=1}^n y_{ji}(t) \\[1mm] 
y_{ji}(t) = w_{ji}\,\sigma(a_{j}x_j(t)\,{+}\,b_{j})\,(e_{i} - x_i(t)) 
\end{array}
\label{eq:na-altc}
\end{equation}
If one takes $e_{li}$ to be zero and $a_j$ to be one, NA-ALTCs are the same as NA-ACT-RNNs except for the gating term $e_{ji}\,{-}\,x_i(t)$. As discussed above, this term makes NA-ALTCs linear systems with state dependent coefficients, which not only makes them more interpretable, but allows the application of state-dependent Ricatti equations in the automatic synthesis of nonlinear controllers~\cite{sdre2008}.

In our experiments, we found that CT-RNNs where the activation function is a hyperbolic tangent, has very similar convergence and learning-accuracy properties with LTCs, where the activation function is a sigmoid. However, hyperbolic tangents fail to capture the opening degree of synaptic channels and their associated polarity, the same way gated sigmoids do: sigmoids accurately capture this degree,  and the difference of potential the gating polarity.

\begin{figure*}[!ht]
    
  \includegraphics[width=.8\textwidth]{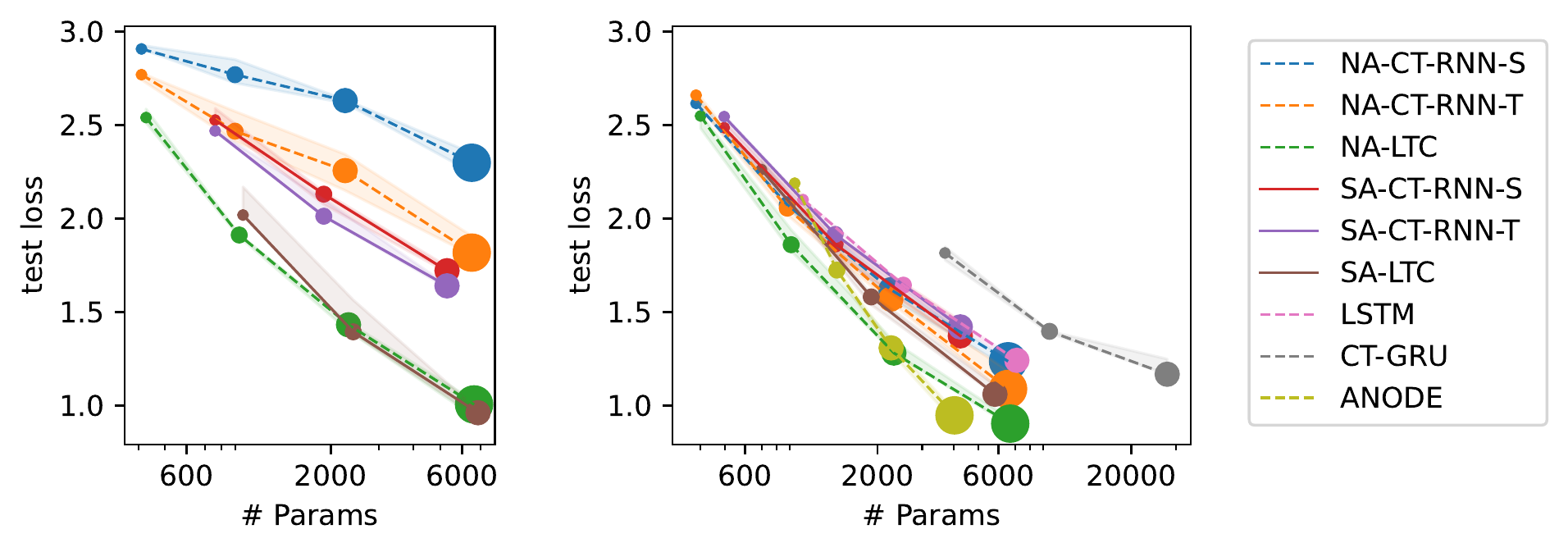}
  \centering
  \vspace*{-2ex}
  \caption{Results for the Walker2d kinematics-learning experiments. Left: synaptic inputs, Right: linear inputs. The size of the marker dots represent the number of neurons (or cells in case of LSTMs): 8, 16, 32 or 64 (from smallest to largest).}
    \label{fig:walker-plt}
\vspace*{-1ex}
\end{figure*}

\begin{figure*}[!ht]
    \centering
  \includegraphics[width=.8\textwidth]{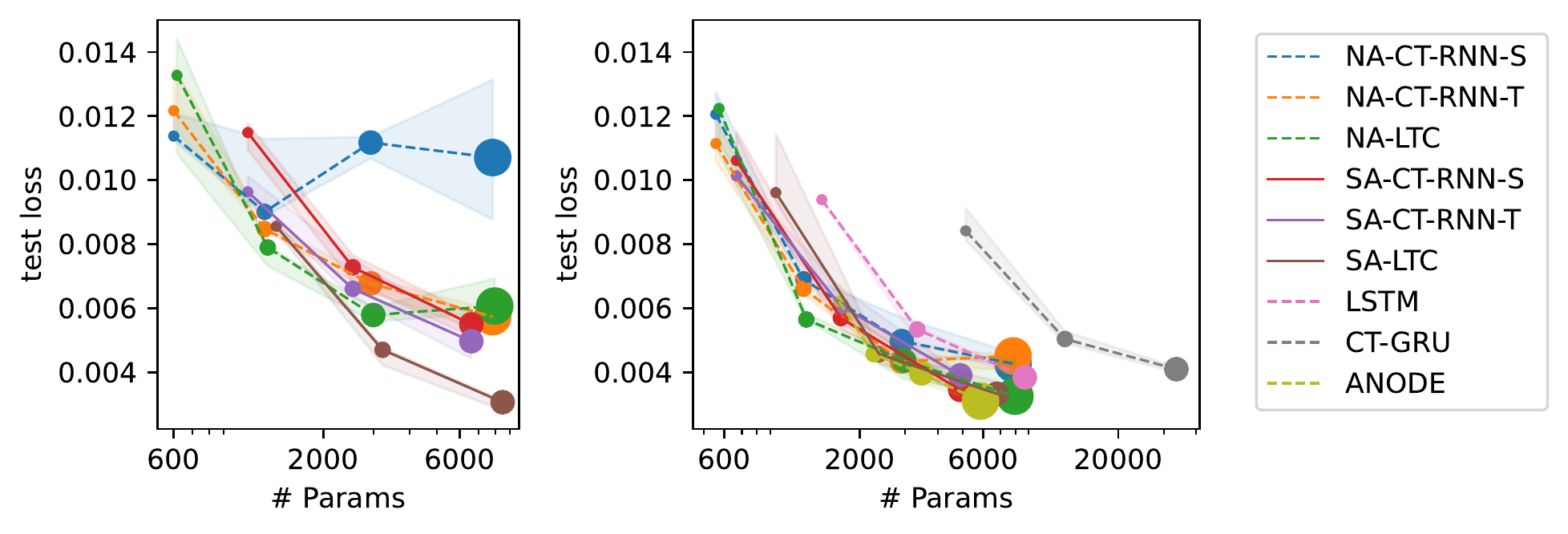}
  \vspace*{-2ex}
  \caption{Results for the Half-Cheetah kinematics-learning experiments. Left/Right and marker size as before. }
    \label{fig:cheetah-plt}
 \vspace*{-3ex}
\end{figure*}

\subsubsection{General case.}
Like CT-RNNs, LTCs have in general associated a time-varying input signal $u$. As for CT-RNNs, we consider both a sigmoidal-input version and a linear-input version.  In scalar form, the first can be written as below:
\[
\begin{array}{c}
C\,\dot{x}_i(t) = w_{li}(e_{li}-x_i(t)) + \sum_{j=1}^n\,y_{ji}(t) + \sum_{j=1}^m\,z_{ji}(t)
\end{array}
\]
\vspace*{-4mm}
\begin{equation}
\label{eq:sa-ltc-si}
\begin{array}{c}
y_{ji}(t) = w_{ji}\,\sigma(a^x_{ji}x_j(t)+b^x_{ji})(e_{ji}-x_i(t))\\[1mm] 
z_{ji}(t) = v_{ji}\,\sigma(a^u_{ji}u_j(t)+b^u_{ji})(e_{ji}-x_i(t))
\end{array}
\end{equation}

\noindent{}An SA-LTC with 32 neurons, connected in an all to all fashion and with 32 inputs, packs a total of 8288 parameters, whereas an SA-CT-RNN packs only 6176 parameters. This roughly corresponds to an NA-CT-RNN with 63 neurons, which packs a total of 8253 parameters.

The linear-input version of SA-LTCs is very similar. Formally, it is described as below:
\[
\begin{array}{c}
C\,\dot{x}_i(t) = w_{li}(e_{li}-x_i(t)) + \sum_{j=1}^n\,y_{ji}(t) + \sum_{j=1}^m\,z_{ji}(t)
\end{array}
\]
\begin{equation}
\label{eq:sa-ltc-li}
\begin{array}{c}
y_{ji}(t) = w_{ji}\,\sigma(a^x_{ji}x_j(t)+b^x_{ji})(e_{ji}-x_i(t))\\[1mm] 
z_{ji}(t) = v_{ji}\,u_j(t)
\end{array}
\end{equation}

\noindent{}An SA-LTC with 32 neurons, connected in an all to all fashion and with 32 inputs, packs 5216 parameters. This roughly corresponds to a linear-input NA-CT-RNN with 51 neurons, which packs a total of 5355 parameters.

\subsubsection{\em Neural activation.}
Like in the autonous case, one can also consider NA-LTCs, where all outgoing synapses of a neuron have the same activation parameters. For the sigmoidal-input version one obtains the following equations: 
\[
\begin{array}{c}
C\,\dot{x}_i(t) = w_{li}(e_{li}-x_i(t)) + \sum_{j=1}^n\,y_{ji}(t) + \sum_{j=1}^m\,z_{ji}(t)
\end{array}
\]
\begin{equation}
\label{eq:na-ltc-si}
\begin{array}{c}
y_{ji}(t) = w_{ji}\,\sigma(a^x_{j}x_j(t)+b^x_{j})(e_{i}-x_i(t))\\[1mm] 
z_{ji}(t) = v_{ji}\,\sigma(a^u_{j}u_j(t)+b^u_{j})(e_{i}-x_i(t))
\end{array}
\end{equation}

\noindent{}The linear-input version is similar, but in this case $z_{ji}(t) = v_{ji}\,u_j(t)$. As for ALTCs, NA-LTCs are very similar to NA-CT-RNNs, with the exception of the gating term $e_{ji}\,{-}\,x_i(t)$. As discussed before, one can get rid of gating, by using a hyperbolic-tangent activation function, with its associated loss of linearity and biophysical meaning.

LTCs are universal approximators~\cite{hasani2020interpretable,ltc2021}. This is true for both their autonomous and general form,  and for synaptic and linear inputs.

\begin{figure*}[!ht]
    \centering
  \includegraphics[width=.8\textwidth]{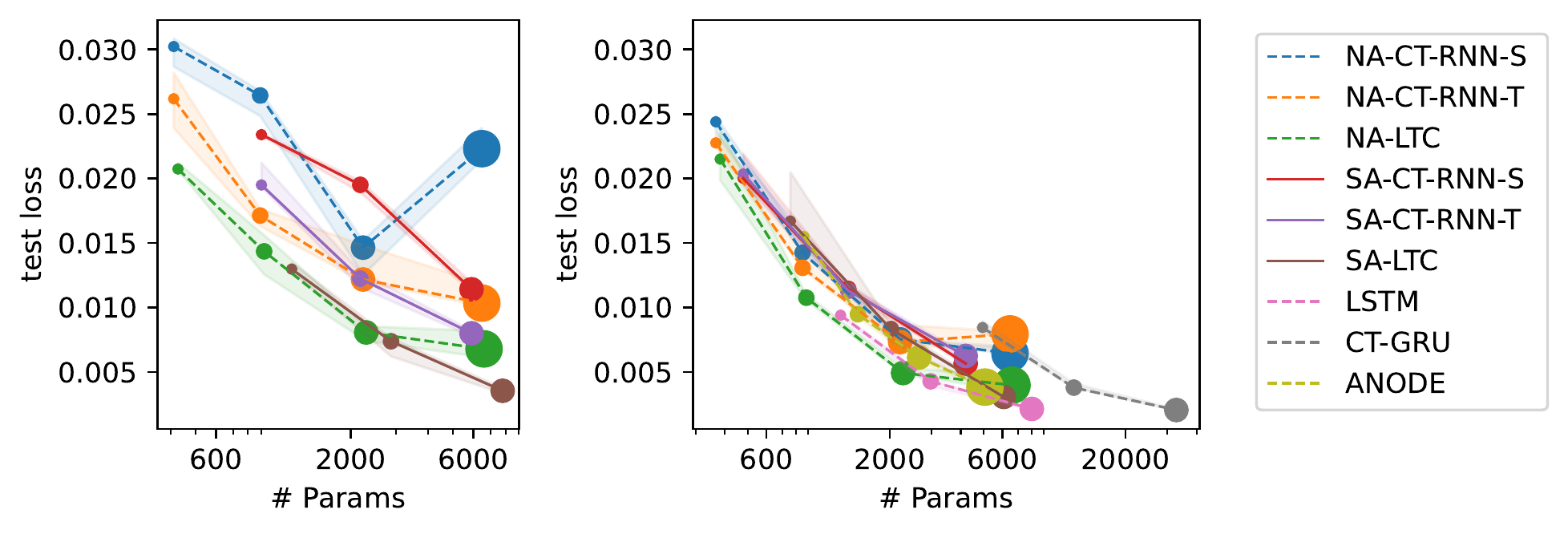}
    \vspace*{-2ex}
  \caption{Results for the Half-Cheetah Behavioural Cloning modeling experiments. Left/Right and marker size as before. }
    \label{fig:cheetah_act-plt}
    \vspace*{-2ex}
\end{figure*}

\begin{figure*}[!ht]
    \centering
  \includegraphics[width=.8\textwidth]{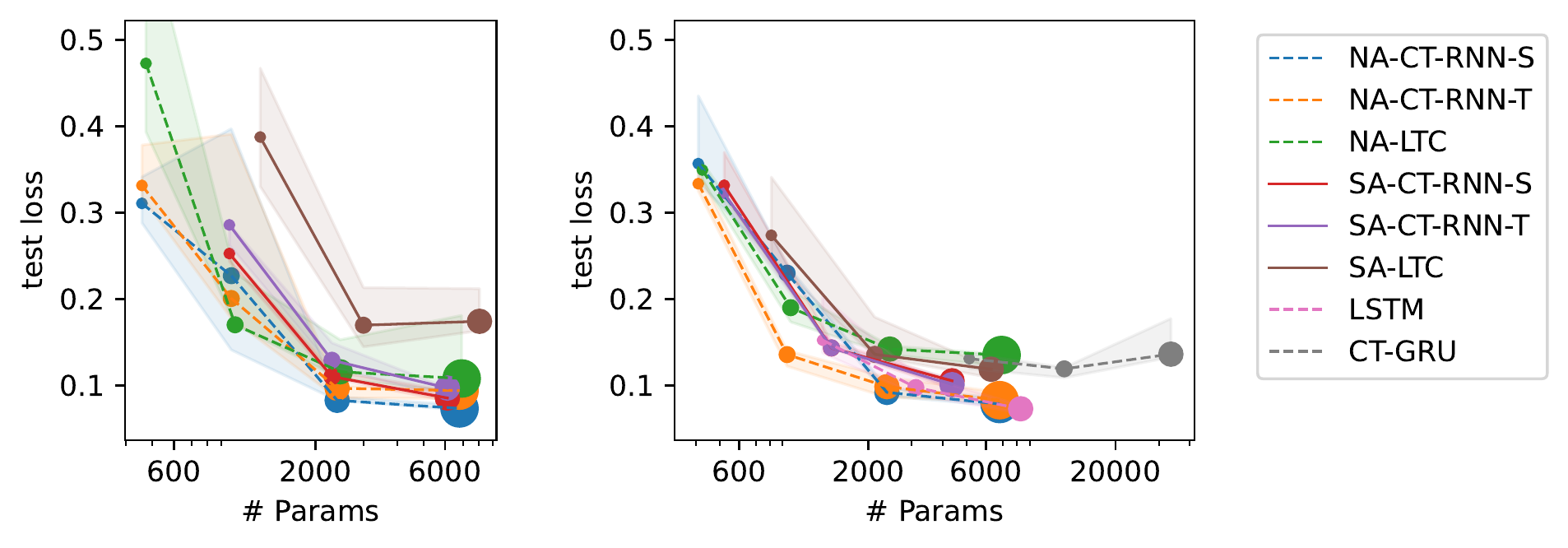}
\vspace*{-2ex}
  \caption{Results for the Sequential MNIST classification experiments. Left/Right and marker size as before. }
    \label{fig:smnist-plt}
\vspace*{-3ex}
\end{figure*}

\section{Experimental Evaluation}
\label{sec:experimentalEvaluation}
\todo{Better emphasize learning technique used (Supervised learnining, hyperparameters, etc.)}
\todo{Explain the mean and variance shown in the plot captions}
\todo{Report convergence rate and timing}
\todo{Report effects of initialization}
\todo{Whats an ablation study?}

\todo{Have more examples. See LTC paper}
\subsection{Sequential model structure}

A three-layered sequential structure was used for all experiments in this section. Let us denote by $u(t)$ the input at time $t$ and by $y_i(t)$ the output of layer $i$ at time $t$. The output of the final layer is the predicted output $\hat y(t) = y_3(t)$.

The first layer maps the inputs to an RNN-layer with either a linear ($y_1(t)\,{=}\, A_{in}u(t)\,{+}\,b_{in}$) or a synaptic (as discussed above) transformation - dubbed \textit{linear} or \textit{synaptic} input mapping, respectively. In case of synaptic input mapping these \textit{sensory} synapses were implemented in accordance with the RNN model used, i.e.~they either used synaptic or neural activation, and did also incorporate the multiplication with a difference of potentials in case of LTCs.

The second layer contains the RNN cells and its output is computed by employing an ODE-solver. The actual ODE being solved is determined by the model type. The different variants are explained in the preceding section. Unlike conventional NeuralODEs, the RNN cells in our model retain a state (and consequently information) after each time-step $t$. 

The third layer, irrespective of the specific model type used, maps the final RNN state $y_2(t)$ to the output vector $y_3(t)$ in a linear fashion, that is, $y_3(t)\,{=}\,A_{out}\, y_2(t)\,{+}\,b_{out}$.

Neural and synaptic dynamics were implemented as py\-torch-lightning modules, ensuring re-usability and portability across different devices such as CPU and GPU. Upon instantiating the module, the desired model type is specified and the parameters are initialized accordingly. Initialization bounds were taken from~\cite{ltc2021}, and are given in the Supplementary Materials. In order to reduce the parameter space, some parameters were fixed at some value and were not subject to training through backpropagation. 

Since the resulting system of ODEs is stiff, the choice of the  ODE-solver has a strong impact on the performance. We chose the explicit Euler solver with 10 unfolds for all the experiments in this paper, as it gave good enough accuracy with low time-complexity, compared to more sophisticated solvers such as Runge-Kutta methods (rk4 or dopri5).

The chain of computations for a layer of synapses with neural and synaptic activation is shown in Figure ~\ref{fig:synaptic} top and bottom respectively. Here, $x$ and $y$ are the state of the pre-synaptic and post-synaptic neuron respectively. LTC-RNNs extend CT-RNN synapses by multiplying the activation with a difference-of-potential term (bottom part of the figures). 
Synaptic activation was realised by extending vectors to matrices throughout the computation graph while also replacing each corresponding matrix-multiplication by the element-wise (or Hadamard) product. Particularly, this means that intermediary results are also represented as matrices while they are vector-valued in case of neural activation.


\subsection{Robotic Experiments}
To explore the parameter-packing and the linear-gating benefits of biophysical synapses we conducted four supervised-learning time-series experiments: Walker2d prediction, Half-Cheetah prediction, Half-Cheetah behavioural cloning, and Sequential-MNIST classification. 

The parameter-packing benefits are evaluated by using CT-RNNs and LTCs with 8, 16, 32, and 64 neurons (connected all-to-all) in the hidden layer, with both neural and synaptic activation, and with both synaptic and linear inputs. 
The gating benefits are evaluated by using CT-RNNs with both sigmoid and tanh activation functions. 

We also provide results for LSTMs, CT-GRUs and ANODEs with 10 augmenting dimensions. We use a linear input mapping as this is the default for the first two. By lacking a stabilization term, the simple NeuralODEs discussed before perform worse than CT-RNNs, and we do not show them in our results. 

Although our results are mainly for robotic control and they are restricted to LTCs and CT-RNNs, we claim that they are applicable to any feedforward or recurrent ANN. 

The CT-RNN and LTC models tested are abbreviated as NA-CT-RNN and NA-LTC for neural activation, and SA-CT-RNN and SA-LTC for synaptic activation. CT-RNNs used either a sigmoidal or a hyperbolic-tangent activation, marked with the suffix S and T, respectively. LTCs always used a sigmoidal activation, because a hyperbolic tangent not only fails to capture the biophysics of synapses, but also renders the network very unstable. For all models we did experiments with both linear- and synaptic-input mappings, as the latter more closely capture sensory neurons. All experiments used the Adam optimizer~\cite{kingma2014adam}.


\subsubsection{Learning Walker-2D kinematics.}
This robotic task is in\-spi\-red by the physics simulation in~\cite{rubanova2019latent}, and implemented by using the gym environment in~\cite{brockman2016openai}. It evaluates how well various RNNs are suited to learn kinematic dynamics. 

To create the training dataset, we used a non-recurrent policy, pretrained via proximal
policy optimization (PPO) \cite{schulman2017}, and the Rllib \cite{rllib} reinforcement learning framework. To increase the task complexity, we used the pretrained policy at 4 different training stages (between 500 to 1200 PPO iterations).
We then collected 17-dimensional observation vectors, performing 400 rollouts of 1000 steps each on the Walker2d-v2 OpenAI gym environment and the MuJoCo physics engine \cite{todorov2012mujoco}. Note that there is no need to include the actions in the training set, because the policy is deterministic. We used 15\% percent of the dataset for testing, 10\% percent for validation and the rest for training.

We aligned the rollouts into sequences of length 20 and then trained each of the models three times for 200 epochs. This was done for 8, 16, 32, and 64 RNN cells. 

Figure~\ref{fig:walker-plt} shows for each model the median test loss and its min and max values, for three runs, with respect to the number of neurons, and the associated number of parameters. 

CT-RNNs perform better for the linear input-mapping, whereas SA-LTCs for synaptic input-mapping. The packing benefit of biophysical synapses is seen in the fact that SA-CT-RNNs and SA-LTCs pack essentially as many parameters as NA-CT-RNNs and NA-LTCs, with half of the number of neurons. The gating benefit is exemplified by the fact that CT-RNN-Ts perform better than CT-RNN-Ss. LTCs perform better than CT-RNNs in all instances. LSTMs and CT-GRUs attain (even greater) parameter packing through a more elaborate concept of a structured cell, but not necessarily with greater accuracy, when their number of cells equals the number of neurons of SA-LTCs. Since LSTMs and CT-GRUs have by default a linear input-mapping, they appear only in the right figure. ANODEs, are also shown in the right figure. They perform comparable to NA-LTCs.

%

\subsubsection{Learning Half-Cheetah kinematics.}
Similar to the Wal\-ker-2D, we learned the kinematics of the Half-Cheetah.

For this experiment we collected 100 rollouts with a controller that was trained using Truncated Quantile Critics (TQC) \cite{tqc}. Just a single version provided by the stable-baselines zoo \cite{rl-zoo} was used this time, making the task relatively easier than the previous one. Again, each rollout is composed of a series of 1000 datapoints consisting of a 27-dimension observation vector generated by the MuJoCo physics engine and a 6-dimension action vector that is produced by the controller. The same data was used in the following two different tasks:

\begin{enumerate}
    \item {\em Kinematics modeling.} Predicting the next observation after having seen 20 preceding observations. The action vectors are not used for this task since the observations serve both as inputs and as labels.
    
    \item {\em Behavioural cloning.} Predicting the next action after having seen 20 preceding observations. In this task the observations serve as inputs while the actions are the labels.
\end{enumerate}

Figure~\ref{fig:cheetah-plt} shows the results for the Half-Cheetah kinematic modeling, and Figure~\ref{fig:cheetah_act-plt} the ones for Half-Cheetah behavioral cloning. The median, min and max test loss are represented as before. The results in both figures follow a very similar pattern as the ones for Walker-2D. However, in this case the benefits of CT-RNN-Ts are evident only for the synaptic input-mapping. LTCs remain more performant.
In Figure~\ref{fig:cheetah_act-plt}, LTSMs and CT-GRUs have a slightly better accuracy compared to SA-LTCs, at the expense of more parameters.

\subsubsection{Sequential-MNIST classification.}
The MNIST dataset consists of 70,000 gray-scale images of 28$\times$28 pixels, containing hand-written digits~\cite{MNISTHandwrittenDigit}. In order to make this task as a sequential one, the images are transformed into sequences of length 28, by taking each row vector, as an input in time. The desired output is a one-hot encoded vector representing integers from 0 to 9. Consequently, a cross-entropy loss was used when training the models. 

The results shown in Figure~\ref{fig:smnist-plt} are as before, the median, min and max test loss of three runs each. They follow a similar pattern with the previous figures, but the LTCs with synaptic inputs are less stable, and fail to properly converge for the largest number of neurons. The best accuracy is attained by NA-CTRNNs and LSTMs.

\section{Discussion and Conclusion}
\label{sec:discussion}
\todo{Don't make generalisation claims without evidence}
The main goal of this paper was to investigate the synaptic-activation and linear-gating benefits of biophysical synapses, as they occur in LTCs. To this end, we asked:
\begin{itemize}
    \item What happens if one uses neural activation in LTCs?
    \item What happens if linear gating is dropped in LTCs?
\end{itemize}

This resulted in two versions of LTCs, and four versions of CT-RNNs, with either linear or synaptic input, and with sigmoid or tanh activation, respectvely. We thoroughly examined the accuracy and parameter-packing ability of these networks, for an increasing number of neurons, and compared them to those of ANODEs, LSTMs, and CT-GRUs. 


We observed that LTCs and CT-RNNs with synaptic activation achieve essentially the same accuracy and parameter packing, for half of the number of neurons, as LTCs and CT-RNNs with neural activation. The linear gating of LTCs further improved this accuracy. We also observed that the accuracy and packing benefits of LTCs is comparable to those of cells in LSTMs. However, the latter rely on a much more elaborate concept of a structured cell. 

We claimed that the benefits of biophysical synapses apply to any ANN. However, we showed them explicitly for LTCs and CT-RNNs, only. Hence, the full version of this paper would have to substantiate this claim. For example, for feed-forward CNNs, one could use a standard CNN base, and consistently replace its neural activations with synaptic ones. Similarly, for LTSMs and CT-GRUs, one could make their recurrent connections synaptic, by using a tanh-activation for each synapse, instead of one for each cell.

For clarity, we kept NeuralODEs as simple as possible, by confining their right-hand-side transformation to one layer. However, one could have used more powerful transformations, which might have led to better results. Nevertheless, we think that the discussed benefits would still apply. 
%

Finally, biophysical synapses may better support sparse networks too, as it was claimed in~\cite{ncp2020}. However, for obvious space reasons, a thorough investigation of this claim had to be postponed to future work.


\bibliography{references}

\section*{Appendix}

Our implementation is using the \texttt{torchdyn} package \cite{poli2020torchdyn} which in turn is based on \texttt{pytorch-lightning}. It comprises several different ODE-solvers supporting automatic differentiation. Since the package by itself was intended for implementing autonomous NeuralODEs, a small trick was necessary for creating CT-RNNs and LTC-RNNs that receive inputs. Autonomous NeuralODEs encapsulate a non-trivial function (such as a neural network) that is used for computing the derivative at each ODE-solver step $\frac{dx}{dt}=F(x)$. For achieving the non-autonomous behavior $\frac{dx}{dt}=F'(x,u)$ the input was appended to the state at each step when passed to the ODE-solver and it's derivative was set to zero $[\frac{dx}{dt}, 0] = F([x,u])$. Consequently, the solution computed by the ODE-solver amounts to being the unaltered input appended to the desired final state.

\begin{equation}
\begin{array}{c}
C\,\dot{x}_i(t) = w_{li}\,(e_{li}-x_i(t)) + \sum_{j=1}^n y_{ji}(t) \\[1mm] 
y_{ji}(t) = w_{ji}\,\sigma(a_{ji}x_j(t)\,{+}\,b_{ji})\,(e_{ji} - x_i(t)) 
\end{array}
\label{eq:sa-altc}
\end{equation}

\begin{table*}[!bh]
\centering
\caption{Results for the INSERT experiments}
\label{tab:cheetah_self}
\begin{tabular}{lcccccccc}
\toprule
              Model & \multicolumn{2}{c}{n = 8} & \multicolumn{2}{c}{n = 16} & \multicolumn{2}{c}{n = 32} & \multicolumn{2}{c}{n = 64} \\
                    & \# Par & MSE $\times 10^{-2}$ & \# Par & MSE $\times 10^{-2}$ & \# Par & MSE $\times 10^{-2}$ & \# Par & MSE $\times 10^{-2}$ \\
\midrule
              ANODE &   1663 &       $0.62\pm 0.04$ &   2263 &       $0.46\pm 0.01$ &   3463 &       $0.40\pm 0.02$ &   5863 &       $0.31\pm 0.01$ \\
             CT-GRU &   5139 &       $0.84\pm 0.05$ &  12427 &       $0.50\pm 0.02$ &        &                    - &        &                    - \\
               LSTM &   1427 &       $0.94\pm 0.00$ &   3339 &       $0.53\pm 0.02$ &   8699 &       $0.38\pm 0.00$ &        &                    - \\
\midrule
  NA-CT-RNN  linear &    555 &        $1.10\pm nan$ &   1211 &        $0.60\pm nan$ &   2907 &       $0.41\pm 0.01$ &   7835 &       $0.45\pm 0.02$ \\
NA-CT-RNN  synaptic &    601 &       $1.04\pm 0.04$ &   1249 &       $0.58\pm 0.02$ &   2929 &       $0.45\pm 0.08$ &   7825 &       $0.57\pm 0.04$ \\
\midrule
     NA-LTC  linear &    571 &       $1.15\pm 0.07$ &   1243 &       $0.56\pm 0.01$ &   2971 &        $0.40\pm nan$ &   7963 &       $0.32\pm 0.07$ \\
   NA-LTC  synaptic &    617 &        $1.08\pm nan$ &   1281 &        $0.67\pm nan$ &   2993 &       $0.52\pm 0.01$ &   7953 &       $0.61\pm 0.06$ \\
\midrule
  SA-CT-RNN  linear &    667 &       $1.01\pm 0.02$ &   1691 &       $0.54\pm 0.00$ &   4891 &        $0.36\pm nan$ &        &                    - \\
SA-CT-RNN  synaptic &   1091 &       $0.88\pm 0.04$ &   2539 &        $0.61\pm nan$ &   6587 &       $0.50\pm 0.04$ &        &                    - \\
\midrule
     SA-LTC  linear &    947 &        $0.93\pm nan$ &   2379 &       $0.46\pm 0.02$ &   6779 &       $0.33\pm 0.03$ &        &                    - \\
   SA-LTC  synaptic &   1371 &       $0.82\pm 0.03$ &   3227 &       $0.47\pm 0.03$ &   8475 &        $0.28\pm nan$ &        &                    - \\
\bottomrule
\end{tabular}
\end{table*}

Table \ref{tab:params} shows the initialization ranges used for the LTC layer parameters taken from previous work. Since LTC dynamics tend to diverge rapidly if leaving the parameters unconstrained, the following conditions were enforced during training: $C \geq 0$,  $w \geq 0$, $a \geq 0$. These constraints are a direct consequence of the capacitor equation Eq. \ref{eq:sa-altc} from which LTCs are derived wherein neither conductance ($w$), capacitance ($C_m$) nor synaptic gating ($a$) may be negative. Similar constraints are also assumed in the proof of universal approximation found in \cite{ltc2021}.

\begin{table}
    \centering
    \caption{Parameters of the recurrent layer and their initialization bounds. $^*$ only present in LTC-RNNs.}
    \begin{tabular}{llc}
    \toprule
         \textbf{Name} & \textbf{Description} & \textbf{Initialization}\\
    \midrule
        \multicolumn{3}{l}{\textit{synaptic parameters}}\\
        $C\ ^*$ & Membrane capacitance & 1 (fixed)\\
        $w$  & Synaptic strength & 0.01 - 1.0\\
        $b$ & Synaptic midpoint & 0.3 - 0.8\\
        $a$ & Synaptic slope & 3 - 8\\
        $e\ ^*$  & Reversal potential & -1 or +1\\
    \midrule
        \multicolumn{3}{l}{\textit{cell body parameters}}\\
        $e_l$  & Resting potential & 0 (fixed)\\
        $w_l$  & Leakage conductance & 0.01 - 1.0\\
        
    \bottomrule
    \end{tabular}
    \label{tab:params}
\end{table}

\subsection{Type-descriptors}\label{ctrnn-type-descriptors}

The form of the derivative used whithin the NeuralODE is configured by
passing a \emph{type descriptor} which is a string of the form:

\vskip 1em
ctrnn\textit{\_[w-mode]act[factor][rec-type]\_in-mode\_[lis]}

\vskip 1em

In any case the derivative is computed using the following formula: $\dot x = Synapses(x, x) + Input(u,x)- decay(x)$. The type descriptor determines how the individual terms are calculated.

\begin{table}[h]
 \centering
 \caption{Type descriptor components and resulting formulas}
\begin{tabular}{ll}
\toprule
\textbf{w-mode} & \(Synapses(x, y)\)\tabularnewline
\midrule
None & \(a\ \sigma(wx+b)*Factor(y)\)\tabularnewline
r & \(w\ \sigma(x+b)*Factor(y)\)\tabularnewline
v & \(w\ \sigma(a(x+b))*Factor(y)\)\tabularnewline
\midrule
\textbf{act}& \(\sigma(x)\)\tabularnewline
\midrule
'sigm' & $1/(1+e^{-x})$\\
'tanh' & $tanh(x)$\\
\midrule
\textbf{factor}& \(Factor(x)\)\tabularnewline
\midrule
\textit{None} & \(1\)\tabularnewline
'*' & \((1-x)\)\tabularnewline
'+' & \((e-x)\)\tabularnewline
\midrule
\textbf{rec-type}&\tabularnewline
\midrule
\textit{None} & neuronal activation\tabularnewline
's' & synaptic activation\tabularnewline
\midrule
\textbf{in-mode} & \(Input(u,x)\)\tabularnewline
\midrule
'linear' & \(Iu + b_i\)\tabularnewline
'synaptic' & \(Synapses(u,x)\)\tabularnewline
\midrule
\textbf{lis} & \(decay(x)\)\tabularnewline
\midrule
\textit{None} & \(\tau x\)\\
 'lis' & \(\tau (x-x_0)\)\tabularnewline
\bottomrule
\end{tabular} 
\end{table}

Sizes of \(E\), \(b\) and \(\alpha\) depend on the \emph{recurrence
type} used:

\begin{table}[H]  
\centering 
\begin{tabular}{ll}
\toprule
\textbf{Recurrence type} & \textbf{Shape of parameters}\tabularnewline
\midrule
neuronal (\textit{None})& (\emph{model\_size})\tabularnewline
synaptic ('s') & (\emph{model\_size}, \emph{model\_size})\tabularnewline
\bottomrule
\end{tabular} \end{table}

The different \textit{w-mode}s allow for adjusting the way Synapses are parameterized, \textit{act} determines the activation function used, the \textit{factor} distinguishes LTCs from CT-RNNs, \textit{rec-type} is used for switching between \textit{Neuronal} and \textit{Synaptic activation}, \textit{in-mode} determines the input mode and \textit{lis} is used when the initial state (= leakage potential) should be learnable.

The recurrence type determines how recurrenct connections are
implemented: 

\begin{itemize}
\item 
\emph{neuronal}: corresponds to the ctrnn model.
Connections can be thought of being only of linear nature, all incoming
synapses of a particular cell share the same bias and summation happens
before applying the activation function.
\item
  \emph{synaptic}: each synapse has a separate bias \(b\), scale \(w\)
  and reference potential \(E\). An additional summation step happens
  after calculating individual synaptic activations.
  \(\dot x = \sum [Synapses(x,x)]+Input(u,x) -decay(x)\)
\end{itemize}

\subsection{Examples}

\begin{itemize}
    \item Vanilla CT-RNN: \texttt{ctrnn\_vtanh\_linear}
    \item SA-LTC: \texttt{ctrnn\_vsigm+s\_synaptic}
\end{itemize}

\end{document}